\documentclass[preprint]{article}

\usepackage[nonatbib]{neurips_2020}

\usepackage[utf8]{inputenc} 
\usepackage[T1]{fontenc}    
\usepackage{hyperref}       
\usepackage{url}            
\usepackage{booktabs}       
\usepackage{amsfonts}       
\usepackage{nicefrac}       
\usepackage{microtype}      

\usepackage{amsmath, amssymb}
\usepackage{mathtools}
\usepackage{amsthm}
\usepackage{xcolor}
\usepackage{mathrsfs} 
\usepackage{xfrac} 
\usepackage[sort&compress]{natbib}

\newcommand{\bsX}{\boldsymbol{X}}

\newcommand{\bsm}{\boldsymbol{m}}

\newcommand{\bsy}{\boldsymbol{y}}

\newcommand{\bsbeta}{\boldsymbol{\beta}}
\newcommand{\bsx}{\boldsymbol{x}}
\newcommand{\bszero}{\boldsymbol{0}}

\newcommand{\bbeta}{\boldsymbol{\beta}}

\newcommand{\bfI}{\mathbf{I}}

\newcommand{\Prob}{\mathbb{P}}

\newcommand{\Exp}{\mathbb{E}}

\newcommand{\bbR}{\mathbb{R}}

\newcommand{\risk}{\mathcal{R}}
\newcommand{\class}[1]{\mathcal{#1}}

\newcommand{\eqdef}{\vcentcolon=}

\newcommand{\parent}[1]{\left( #1 \right)}

\newcommand{\enscond}[2]{\left\{ #1 \, : \, #2\right\}}

\newcommand{\scalar}[2]{\left\langle #1, #2\right\rangle}

\newcommand{\eqd}{\stackrel{\text{d}}{=}}

\DeclareMathOperator*{\supp}{supp}

\newcommand{\ie}{{\em i.e.,~}}

\newtheorem{theorem}{Theorem}[section]
\newtheorem{definition}[theorem]{Definition}
\newtheorem{example}[theorem]{Example}
\newtheorem{proposition}[theorem]{Proposition}
\newtheorem{lemma}[theorem]{Lemma}
\newtheorem{assumption}[theorem]{Assumption}

\newtheorem*{theorem*}{Theorem}
\newtheorem*{proposition*}{Proposition}
\newtheorem*{lemma*}{Lemma}

\definecolor{green}{rgb}{0.0, 0.5, 0.0}
\definecolor{red}{rgb}{0.8, 0.0, 0.0}
\definecolor{blue}{rgb}{0.01, 0.28, 1.0}
\definecolor{yellow}{rgb}{0.98, 0.93, 0.36}

\usepackage[textwidth=2.0cm, textsize=tiny]{todonotes} 

\usepackage{ wasysym }

\definecolor{orange}{rgb}{0.76, 0.23, 0.13}

\hypersetup{
     colorlinks = true,
     linkcolor = blue,
     anchorcolor = blue,
     citecolor = blue,
     filecolor = blue,
     urlcolor = blue
}

    \newlength{\leftstackrelawd}
    \newlength{\leftstackrelbwd}
    \def\leftstackrel#1#2{\settowidth{\leftstackrelawd}%
    {${{}^{#1}}$}\settowidth{\leftstackrelbwd}{$#2$}%
    \addtolength{\leftstackrelawd}{-\leftstackrelbwd}%
    \leavevmode\ifthenelse{\lengthtest{\leftstackrelawd>0pt}}%
    {\kern-.5\leftstackrelawd}{}\mathrel{\mathop{#2}\limits^{#1}}}

\usepackage{ stmaryrd }

\title{An example of prediction which complies with Demographic Parity and equalizes group-wise risks in the context of regression}

\author{%
  Evgenii Chzhen\\
  LMO, Universit\'e Paris-Saclay\\
  CNRS, INRIA\\
  \texttt{evgenii.chzhen@math.u-psud.fr}
  \And
  Nicolas Schreuder\\
  CREST, ENSAE\\
  Institut Polytechnique de Paris\\
  \texttt{nicolas.schreuder@ensae.fr}
}

\begin{document}

\maketitle

\begin{abstract}
Let $(\bsX, S, Y) \in \bbR^p \times \{1, 2\} \times \bbR$ be a triplet following some joint distribution $\Prob$ with feature vector $\bsX$, sensitive attribute $S$ , and target variable $Y$.
The Bayes optimal prediction $f^*$ which does not produce Disparate Treatment is defined as $f^*(\bsx) = \Exp[Y \mid \bsX = \bsx]$.
We provide a non-trivial example of a prediction $\bsx \to f(\bsx)$ which satisfies two common group-fairness notions: Demographic Parity and Equal Group-Wise Risks
\begin{align*}
    (f(\bsX) \mid S = 1) &\stackrel{d}{=} (f(\bsX) \mid S = 2)\enspace,\\
    \Exp[(f^*(\bsX) - f(\bsX))^2 \mid S = 1] &= \Exp[(f^*(\bsX) - f(\bsX))^2 \mid S = 2]\enspace.
\end{align*}
To the best of our knowledge this is the first explicit construction of a non-constant predictor satisfying the above.
We discuss several implications of this result on better understanding of mathematical notions of algorithmic fairness.
\end{abstract}


\section{Introduction}
\label{sec:intro}

Designing methods that satisfy group-fairness requirements has received a lot of theoretical and empirical attention in recent years \citep{barocas-hardt-narayanan,calmon2017optimized,chierichetti2017fair,Donini_Oneto_Ben-David_Taylor_Pontil18,dwork2018decoupled,hardt2016equality,dwork2012fairness,kilbertus2017avoiding,lum2016statistical,zafar2017fairness,zemel2013learning, agarwal2019fair,lipton2018does,chiappa2020general,gouic2020price,chzhen2020fair}.
Most of the contributions in this direction are concerned with the problem of binary classification, while the regression setup receiving much less attention to this date \citep{agarwal2019fair}.
However, even if the underlying problem at hand has a structure of binary classification, a continuous regression-type output might be more informative in real-world scenarios.

In the literature on algorithmic fairness, it is a standard practice to consider two distinct types of predictions: \emph{fairness through awareness}~\citep{dwork2012fairness} and \emph{fairness through unawareness (without Disparate Treatment)}~\citep{gajane2017formalizing,lipton2018does}.
The former type of prediction allows one to build separate model for each sensitive attribute, while the latter obliges one to fix a single model which is later applied across all groups. 
In the infinite sample regime, assuming that the joint distribution of the observations is known, recent works showed that the problem of regression with fairness through awareness under the \emph{Demographic Parity} constraint shares a strong connection with the problem of Wasserstein barycenters~\citep{gouic2020price,chzhen2020fair}. In particular,~\citet{gouic2020price} derives a closed form expression of fair optimal prediction in the sense of Demographic Parity.
However, very little is known about the predictions which avoid Disparate Treatment and achieve Demographic Parity even in the infinite sample regime. Actually, even the existence of non-trivial regression prediction strategies satisfying the two constraints is unclear.

In this work we make progress towards the mathematical understanding of the latter problem. We make the following contributions: we propose a large family of prediction functions which achieve Demographic Parity without producing Disparate Treatment; we identify a specific function within this class which additionally equalizes the group-wise risks.
Even though the proposed prediction rule achieves several desirable formal group-fairness notions, we argue that this prediction is not suitable for real-world scenarios. 
In contrast, a prediction that is allowed to produce Disparate Treatment can alleviate these drawbacks.
In the context of binary classification, similar conclusions were reached by~\cite{lipton2018does}.

\paragraph{Organization} The rest of this note is organised as follows. We present in Section~\ref{sec:setup_goals} our setup and general goal. In Section~\ref{sec:main_result} we provide a description of a family of prediction rules satisfying the fairness constraints of interest. Finally in Section~\ref{sec:discussion} we discuss a critical flaw of those prediction rules from individual level fairness viewpoint and provide some open questions.
Proofs can be found in Section~\ref{sec:proofs}.

\paragraph{Notation} For a distribution $\mu$ defined on a
measurable space $(X,\mathscr X)$ and a measurable 
map $T:X\mapsto Y$, where $Y$ is another space endowed with
a $\sigma$-algebra $\mathscr Y$, we denote by $T\sharp \mu$
the push-forward measure defined by $(T\sharp \mu)(A) =
\mu\big(T^{-1}(A)\big)$ for all $A\in \mathscr Y$. For two random variables $U, V$ we write $U \eqd V$ to denote their equality in distribution. The standard Euclidean inner product and Euclidean norm in $\bbR^p$ are denoted by $\scalar{\cdot}{\cdot}$ and $\|\cdot\|_2$ respectively.

\section{Setup and general goal}
\label{sec:setup_goals}
Let $(\bsX, S, Y) \in \bbR^p \times \{1, 2\} \times \bbR$ be a triplet following some joint distribution $\Prob$ where $\bsX$ is a feature vector, $S$ a binary sensitive attribute (\textit{e.g.}, gender or race) and $Y$ is a target variable. For $s \in \{1,2\}$, let $\mu_{\bsX | s}$ denote the distribution of the features inside the group $S = s$.
We are interested in finding a mapping between the feature vector and the target variable which is fair in a sense we specify in this section.

The first notion of fairness that we consider restricts the class of predictors to those which do not take as input the sensitive attribute $S$.
\begin{definition}[Disparate Treatment]
    \label{def:DI}
   Any measurable function $f : \bbR^p \to \bbR$ that cannot receive the sensitive attribute $S$ in its functional form does not produce \emph{Disparate Treatment}.
\end{definition}
We note that \citet{gajane2017formalizing} refer to the latter as fairness through unawareness.
This property might be desirable for obvious legal and/or privacy reasons \citep{primus2003equal,barocas2016big,gajane2017formalizing}. However it does not guarantee the prediction to be \emph{statistically independent} from the sensitive attribute $S$ because of correlations between the sensitive attribute $S$ and the feature vector $\bsX$. Indeed, consider the Bayes optimal predictor $\bsx \mapsto f^*(\bsx)$ defined as
\begin{align*}
    f^*(\bsx) = \Exp[Y \mid \bsX = \bsx]\enspace.
\end{align*}
It does not take as input the sensitive attribute and achieves the lowest possible squared risk among predictions avoiding Disparate Treatment.
Yet, the predictor $f^*$ might still promote disparity between sensitive groups if the distributions of features $\bsX$ differ between groups.

To address the above shortcoming, we further restrict the space of possible predictions to those satisfying Demographic Parity (DP)~\citep{calders2009building,calders2013controlling}. 

\begin{definition}[Demographic Parity]
   A predictor $f : \bbR^p \to \bbR$ achieves \emph{Demographic Parity} if
   \begin{align*}
       (f(\bsX) \mid S = 1) \eqd (f(\bsX) \mid S = 2)\enspace.
   \end{align*}
\end{definition}
Such predictors are also said to avoid \emph{Disparate Impact}. This notion of fairness is quite intuitive since it asks the group-wise distributions of the predictions to be the same across all groups. However this probabilistic constraint is not particularly nice to handle and describing explicitly all the functions satisfying this constraint is not an easy  task. Obviously, any constant function satisfies this constraint; but what about functions depending on the feature vector $\bsX$ ? It is not obvious that one can design a non-trivial function $f$ which does not depend on the sensitive attribute $S$ while achieving Demographic Parity. We give two simple scenarios for which we can explicit the class of functions satisfying Demographic Parity.
\begin{example}[Simple case 1]
    \label{ex:simple1}
    Assume that distributions of the features $\bsX$ is the same within each group, \ie
    \begin{align*}
        (\bsX \mid S = 1) \eqd (\bsX \mid S = 2)\enspace.
    \end{align*}
    In this case \emph{any} function $f : \bbR^p \to \bbR$ achieves Demographic Parity. In particular, one can use the Bayes optimal prediction $f^*$.
\end{example}

\begin{example}[Simple case 2]
    Assume that the sensitive attribute $S$ is a deterministic function of the features $\bsX$
    and that the supports of $\mu_{\bsX | 1}$ and $\mu_{\bsX | 2}$ are non-intersecting. Then we can construct two different functions $f_1, f_2 : \bbR^p \to \bbR$ such that
    \begin{align*}
        (f_1(\bsX) \mid S = 1) \eqd (f_2(\bsX) \mid S = 2)\enspace.
    \end{align*}
    Then we define $f(\bsx) \eqdef f_1(\bsx)$ for all $\bsx$ in the support of $\mu_{\bsX | 1}$ and $f(\bsx) \eqdef f_2(\bsx)$ for all $\bsx$ in the support of $\mu_{\bsX | 2}$.
    In this way Demographic Parity is achieved by \emph{one} function $f$, which avoids Disparate Impact.
    In other words, when the sensitive attribute $S$ is a deterministic function of features $\bsX$, using $S$ or not using $S$ in the functional form of the prediction does not change anything.
\end{example}

Those two toy examples enable us to get a better understanding of Demographic Parity; however they are far from sufficient since assuming that the features are distributed the same across groups or that the sensitive attribute is a deterministic function of $\bsX$ is clearly unrealistic in practice. Thus, we would like to be able to cover more scenarios than those listed above.
More formally, the main question that we would like to address is:
\begin{align*}
    &\textbf{Main question: } &&\text{Is there a non trivial prediction strategy $f$ which}\\
    & &&\text{1) avoids Disparate Treatment;}\\
    & &&\text{2) achieves Demographic Parity;}\\
    & &&\text{under minimal assumptions on the distribution of $(\bsX, S, Y)$ ?}
\end{align*}
Let us emphasize that the main mathematical challenge of this question comes from the fact that the sensitive attribute cannot be used in the functional form of the prediction while the prediction must satisfy a constraint depending on the sensitive attribute.
We elaborate more on this issue in the next example.
\begin{example}[Gaussian features]
\label{eq:gauss}
Assume that the feature vector $\bsX \mid S$ is distributed as
\begin{align*}
    (\bsX \mid S = 1) \sim \class{N}(\bsm_1, \bfI),\qquad (\bsX \mid S = 2) \sim \class{N}(\bsm_2, 2\bfI)\enspace,
\end{align*}
with $\bsm_1 \neq \bsm_2$. It is very easy to find a function $g : \bbR^p \times \{1, 2\} \to \bbR$ so that
\begin{align*}
    (g(\bsX, S) \mid S = 1) \eqd (g(\bsX, S) \mid S = 2)\enspace.
\end{align*}
In particular, one can consider group-wise affine predictions: $g(\bsx, 1) = \scalar{\bsbeta_1}{\bsx} + b_1$ and $g(\bsx, 2) = \scalar{\bsbeta_2}{\bsx} + b_2$ with $\bsbeta_1, \bsbeta_2 \in \bbR^p$ satisfying
\begin{align*}
    \scalar{\bsbeta_1}{\bsm_1} + b_1 = \scalar{\bsbeta_2}{\bsm_2} + b_2,\qquad \|\bsbeta_1\|_2 = \sqrt{2}\|\bsbeta_2\|_2\enspace.
\end{align*}
Moreover, the risk-optimal choice of $g$ is also group-wise affine (we elaborate on it later in the text). 
However, there is no non-trivial (\ie nonconstant) affine function $f: \bbR^p \to \bbR$ which achieves Demographic Parity and avoids Disparate Treatment.

Indeed, assume that there exist $\bbeta \in \bbR^p, b \in \bbR$ such that $f(\bsx) = \scalar{\bsbeta}{\bsx} + b$ achieves Demographic Parity. 
Since the features are group-wise Gaussians, then
\begin{align*}
    (f(\bsX) \mid S = 1) \sim \class{N}(\scalar{\bbeta}{\bsm_1} + b, \|\bbeta\|^2_2),\qquad (f(\bsX) \mid S = 2) \sim \class{N}(\scalar{\bbeta}{\bsm_2} + b, 2\|\bbeta\|^2_2)\enspace.
\end{align*}
For the above two distributions to be equal we must set $\bbeta = \bszero$ and the prediction $f \equiv b$ reduces to a trivial constant.
\end{example}
Example~\ref{eq:gauss}  highlights the intrinsic difficulty of the considered question --
even if the distribution of the covariates is group-wise Gaussian and even if the Bayes optimal prediction $f^*(\bsx) = \Exp[Y \mid \bsX = \bsx]$ is affine, there is \emph{no} non-trivial affine prediction rule $f(\bsx) = \scalar{\bsbeta}{\bsx} + b$ achieving Demographic Parity \emph{and} avoiding Disparate Treatment.
Besides, this example demonstrates that learning prediction function without Disparate Impact and Disparate Treatment in an {agnostic learning manner} might be a bad idea.
Indeed, assume the same model as in Example~\ref{eq:gauss} and define $f_{\text{affine}}^{\text{DP}}$ as a solution of
\begin{align*}
    \min_{f : \bbR^p \to \bbR}\enscond{\Exp(Y - f(\bsX))^2}{(f(\bsX) \mid S = 1) \eqd (f(\bsX) \mid S = 2),\quad f \in \class{F}_{\text{affine}}}\enspace,
\end{align*}
where $\class{F}_{\text{affine}} = \{f:\bsx \mapsto  \scalar{\bsbeta}{\bsx} + b;\bsbeta \in \bbR^p, b \in \bbR\}$ -- a recurrent prediction class restriction in the learning literature~\citep{Vapnik_Chervonenkis68}. Following Example~\ref{eq:gauss} we know that $f_{\text{affine}}^{\text{DP}}$ is a trivial constant prediction thus building a data-driven method which performs as well as $f_{\text{affine}}^{\text{DP}}$ is not that relevant. Due to these observations we believe that current fairness definitions should be first examined without the restriction of the predictors.

In this work we provide a large family of prediction functions $\class{F}_{\nnearrow}$, which are parametrized by non-decreasing continuous functions $Q : [0, 1] \to \bbR$. Every function $f_Q \in \class{F}_{\nnearrow}$ avoid Disparate Treatment and achieves Demographic Parity.
Furthermore, we show that the family $\class{F}_{\nnearrow}$ contains a special prediction function  $f_{Q^*}$, which achieves an additional fairness criterion. Namely, it achieves \emph{Equality of Group-Wise Risks} defined below.
\begin{definition}
   A predictor $f : \bbR^p \to \bbR$ achieves the \emph{Equality of Group-Wise Risks} (EGWR) constraint if
   \begin{align*}
       \mathbb{E}[(f^*(\bsX) -f(\bsX))^2 \mid S=1] = \mathbb{E}[(f^*(\bsX) -f(\bsX))^2 \mid S=2]\enspace.
   \end{align*}
\end{definition}
Similar notion of fairness in its relaxed formulation was considered in the context of regression by~\cite{agarwal2019fair}.

Despite three fruitful properties of formal group-fairness requirements achieved by $f_{Q^*}$, we argue that this function fails to satisfy basic principles of fairness and justice.
The main reason for its failure is the avoidance of Disparate Treatment, which forces a prediction to ``guess'' the sensitive attribute of a given feature vector $\bsx \in \bbR^p$.
Such guessing leads to undesirable predictions for individuals $\bsx \in \bbR^p$ with sensitive attribute $S = 1$ but who are more likely to have $S = 2$ and vice-versa.

\paragraph{Fairness through awareness: a reminder} Before going further into the problem, let us provide a short theoretical reminder for the situation when the sensitive attribute $S$ is allowed to be used in the functional form of the prediction, that is, the Disparate Treatment is allowed. More formally, we are interested in finding a prediction $g^* : \bbR^p \times \{1, 2\} \to \bbR$ which is a solution of
\begin{align*}
    \min_{g : \bbR^p \times \{1, 2\} \to \bbR}\enscond{\Exp(Y - g(\bsX, S))^2}{(g(\bsX, S) \mid S = 1) \eqd (g(\bsX, S) \mid S = 2)}\enspace.
\end{align*}
\citet{gouic2020price,chzhen2020fair} showed that under mild additional assumptions\footnote{They assume that for all $s \in \{1, 2\}$ the measure $g(\cdot, s) \sharp \mu_{\bsX | s}$ is continuous and has finite second moment.} on the distribution $\Prob$, the optimal fair prediction $g^*$ can be obtained for all $(\bsx, s) \in \bbR^p \times \{1, 2\}$ as
\begin{align}
    \label{eq:aware}
    g^*(\bsx, s) = \left(p_1 G_{1}^{-1} + p_2G_{2}^{-1}\right) \circ G_s \big(\Exp[Y \mid \bsX = \bsx, S = s]\big)\enspace,
\end{align}
where $p_s = \Prob(S = s)$, $G_s(t) = \Prob\parent{\Exp[Y \mid \bsX, S] \leq t \mid S = s}$, and $G_s^{-1}$ is the generalized inverse of $G_s$ for all $s \in \{1, 2\}$.
In particular, returning to Example~\ref{eq:gauss}, one can show that if $\Exp[Y \mid \bsX, S] = \scalar{\bsbeta^*_S}{\bsX} + b_S$, then the fair optimal prediction $g^*$ is also group-wise affine.
Indeed, we note that under the assumptions of Example~\ref{eq:gauss} it holds that
\begin{align*}
    (\Exp[Y \mid \bsX, S] \mid S = 1) \sim \class{N}(b_1, \|\bsbeta^*_1\|_2^2),\qquad 
    (\Exp[Y \mid \bsX, S] \mid S = 2) \sim \class{N}(b_2, 2\|\bsbeta^*_2\|_2^2)\enspace.
\end{align*}
Denoting by $\Phi$ the cumulative distribution function of the standard Gaussian we can write that $G_1(t) = \Phi(\sfrac{(t - b_1)}{\|\bsbeta^*_1\|_2})$ and $G_2(t) = \Phi(\sfrac{(t - b_2)}{\sqrt{2}\|\bsbeta^*_1\|_2})$. Their inverses can be respectively written as
\begin{align*}
    G^{-1}_1(t) = b_1 + \|\bsbeta^*_1\|_2 \Phi^{-1}(t), \qquad G^{-1}_2(t) = b_2 + \sqrt{2}\|\bsbeta^*_2\|_2 \Phi^{-1}(t)\enspace.
\end{align*}
Substituting these expressions into Eq.~\eqref{eq:aware} and simplifying we get that
\begin{align*}
    g^*(\bsx, 1) &= 
    \scalar{\bsbeta_1^*}{\bsx}\parent{p_1 + p_2\frac{\sqrt{2}\|\bsbeta^*_2\|_2}{\|\bsbeta^*_1\|_2}} + p_1b_1 + p_2b_2\enspace,\\
    g^*(\bsx, 2) &= 
    \scalar{\bsbeta_2^*}{\bsx}\parent{p_2 + p_1\frac{\|\bsbeta^*_1\|_2}{\sqrt{2}\|\bsbeta^*_2\|_2}} + p_1b_1 + p_2b_2\enspace.
\end{align*}
The above highlights that in the case of linear regression model, the predictor $g^* : \bbR^p \times \{1, 2\} \to \bbR$ which minimizes the risk under the Demographic Parity constraint remains affine. We again emphasize that the situation is changed drastically if a prediction is not allowed to produce Disparate Treatment.

\begin{figure}[!t]
\centering
\includegraphics[width=0.45\textwidth]{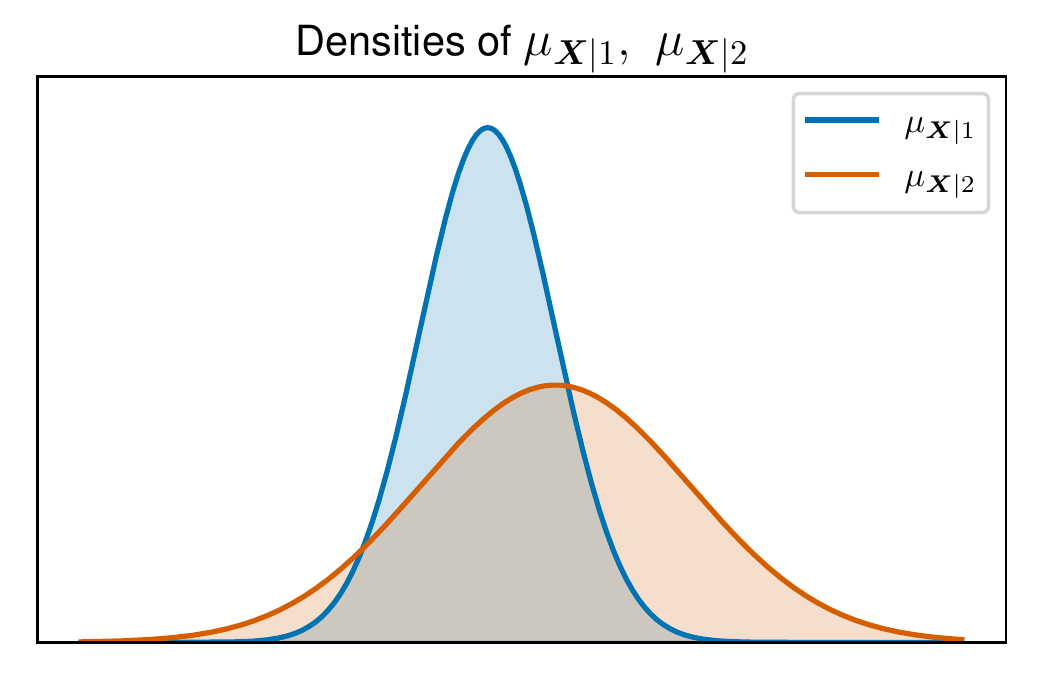}
\includegraphics[width=0.45\textwidth]{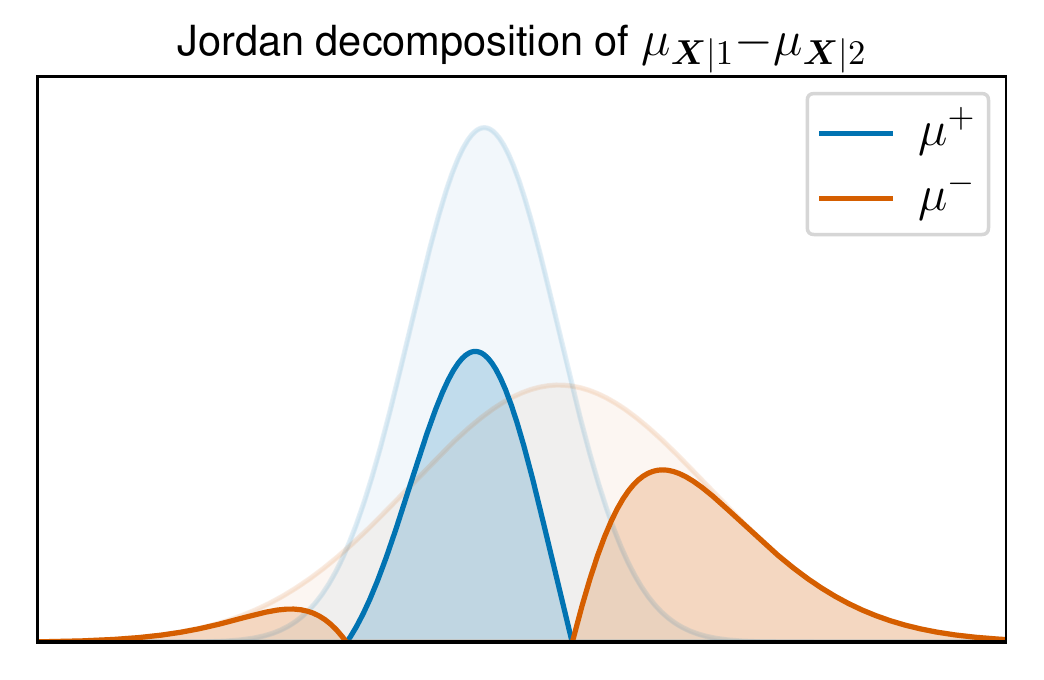}
\caption{Jordan decomposition of a signed measure. (\texttt{Left}) An example of feature distributions within two groups. (\texttt{Right}) Jordan decomposition of the difference $\mu_{\bsX | 1} - \mu_{\bsX | 2}$.}
\label{fig:hahn}
\end{figure}
\section{Description of the family}
\label{sec:main_result}
In this section we present a family of prediction rules, indexed by the set of continuous non-decreasing functions $Q : [0, 1] \to \mathbb{R}$, which achieve Demographic Parity and explicit a function from this family which also satisfies the Equality of Group-Wise Risks constraint.

\paragraph{Jordan decomposition} Consider the signed measure $\mu \coloneqq \mu_{\bsX | 1} - \mu_{\bsX | 2}$, and let $\mu^+, \mu^-$ be its Jordan decomposition, that is $\mu = \mu^+ - \mu^-$ and $\supp(\mu^+) \cap \supp(\mu^-) = \emptyset$.
Unless the supports of $\mu_{\bsX | 1}$ and $\mu_{\bsX | 2}$ are disjoint, the measures $\mu^+$ and $\mu^-$ do not integrate to one, \ie they are not probability measures. However, both $\mu^+$ and $\mu^-$ have the same total mass.
We define $P\mu^{\pm} = \sfrac{\mu^{\pm}}{\mu^{\pm}(\bbR)}$ the projection of $\mu^{\pm}$ on the space of probability measures.
We also define
\begin{align*}
    F_{\pm}(t) \eqdef P\mu^{\pm}\parent{\enscond{\bsx \in \bbR^p}{f^*(\bsx) \leq t}}\enspace,
\end{align*}
the cumulative distribution function of $f^* \sharp P\mu^{\pm}$.

The supports of measures $\mu^+$ and $\mu^-$ have a simple and intuitive interpretation. Note that if $\bsx \in \supp(\mu^+)$, then $\bsx$ is more likely to be a member of the group $S = 1$ and vice versa.
Meanwhile, if $\bsx \in \bbR^p \setminus (\supp(\mu^+) \cup \supp(\mu^-))$ then $\bsx$ can be equally likely coming from $S = 1$ or from $S = 2$.
See Figure~\ref{fig:hahn} for an illustration with univariate covariates.

The rational behind the introduction of the Jordan decomposition of $\mu$ into $\mu^+$ and $\mu^-$ comes from the following simple insight. It says that in order to check Demographic Parity for predictions without Disparate Treatment one only needs to know $\mu^+$ and $\mu^-$ instead of the whole distribution of the covariates $\mu_{\bsX|s}$. This idea is formalized in the next lemma.
\begin{lemma}
\label{lem:signed}
A prediction without Disparate Treatment $f : \bbR^p \to \bbR$ achieves Demographic Parity iff
\begin{align}
    \label{eq:signed_eq}
    f \sharp \mu^+ = f \sharp \mu^-\enspace.
\end{align}
\end{lemma}
\begin{proof}
Set $A_{\square} = \supp(\mu^\square)$ for $\square \in \{\pm\}$ and $A_0 = \bbR^p \setminus (A_+ \cup A_-)$.

($\Rightarrow$) If $f$ achieves Demographic Parity, then $f \sharp \mu_{\bsX | 1} = f \sharp \mu_{\bsX | 2}$. Note that, for all $t \in \bbR$ it holds that
\begin{align*}
    \mu_{\bsX | \square}\enscond{\bsx \in \bbR^p}{f(\bsx) \leq t} =
    \mu_{\bsX | \square}\enscond{\bsx \in A_+}{f(\bsx) \leq t}
    &+ \mu_{\bsX | \square}\enscond{\bsx \in A_-}{f(\bsx) \leq t}\\
    &+ 
    \mu_{\bsX | \square}\enscond{\bsx \in A_0}{f(\bsx) \leq t}\enspace.
\end{align*}
Note that by the definition of $\mu^+$ and $\mu^-$ it holds that
\begin{align*}
    \mu_{\bsX | 1}\enscond{\bsx \in A_0}{f(\bsx) \leq t} = \mu_{\bsX | 2}\enscond{\bsx \in A_0}{f(\bsx) \leq t}\enspace,
\end{align*}
and thus, the condition $f \sharp \mu_{\bsX | 1} = f \sharp \mu_{\bsX | 2}$ implies that for all $t \in \bbR$
\begin{align}
     \mu_{\bsX | 1}\enscond{\bsx \in A_+}{f(\bsx) \leq t} - &\mu_{\bsX | 2}\enscond{\bsx \in A_+}{f(\bsx) \leq t} =\nonumber\\ &\mu_{\bsX | 2}\enscond{\bsx \in A_-}{f(\bsx) \leq t} - \mu_{\bsX | 1}\enscond{\bsx \in A_-}{f(\bsx) \leq t}\enspace.\label{eq:signed2}
\end{align}
The latter is equivalent to Eq.~\eqref{eq:signed_eq}.

($\Leftarrow$) Recall that Eq.~\eqref{eq:signed_eq} is equivalent to Eq.~\eqref{eq:signed2} and that for any $f : \bbR^p \to \bbR$ it holds that
\begin{align*}
    \mu_{\bsX | 1}\enscond{\bsx \in A_0}{f(\bsx) \leq t} = \mu_{\bsX | 2}\enscond{\bsx \in A_0}{f(\bsx) \leq t}\enspace.
\end{align*}
Combining both concludes the proof.
\end{proof}

Note that such an argument would not work if $f$ was allowed to depend on the sensitive attribute.

\paragraph{Prediction rules} Let $Q: [0, 1] \to \bbR$ be any continuous non-decreasing function. Define the following prediction rule
\begin{align*}
\label{eq:predictor}
\tag{\textasteriskcentered}
    f_{Q}(\bsx)
    =
    \begin{cases}
        Q \circ F_+ \circ f^*(\bsx) &\text{if } \bsx \in \supp(\mu^+) \\
        Q\circ F_- \circ f^*(\bsx) &\text{if } \bsx \in \supp(\mu^-)\\
       f^*(\bsx) &\text{if } \bsx \in  \bbR^p \setminus \parent{\supp(\mu^+) \cup \supp(\mu^-)}
    \end{cases}\enspace.
\end{align*}
One should think of $Q$ as a quantile function of some continuous univariate probability measure $\lambda$. In the language of optimal transport the function $Q \circ F_\square \circ f^*(\cdot)$ is the optimal transport map from $f^* \sharp P\mu^\square$ to $\lambda$. An exact theoretical motivation to introduce function $Q$ will be clarified later in the text. 

There are three cases in the above prediction rule:
\begin{enumerate}
    \item $\bsx \in \supp(\mu^+)$, in this case $\bsx$ is more likely associated with $S = 1$.
    \item $\bsx \in \supp(\mu^-)$, in this case $\bsx$ is more likely associated with $S = 2$.
        \item $\bsx \in  \bbR^p \setminus \parent{\supp(\mu^+) \cup \supp(\mu^-)}$, in this case $\bsx$ can be equally likely associated with group $S = 1$ and $S = 2$ and the decision is made in accordance with the Bayes optimal prediction by the analogy with Example~\ref{ex:simple1}.
\end{enumerate}
\begin{figure}[!t]
\centering
\includegraphics[width=0.48\textwidth]{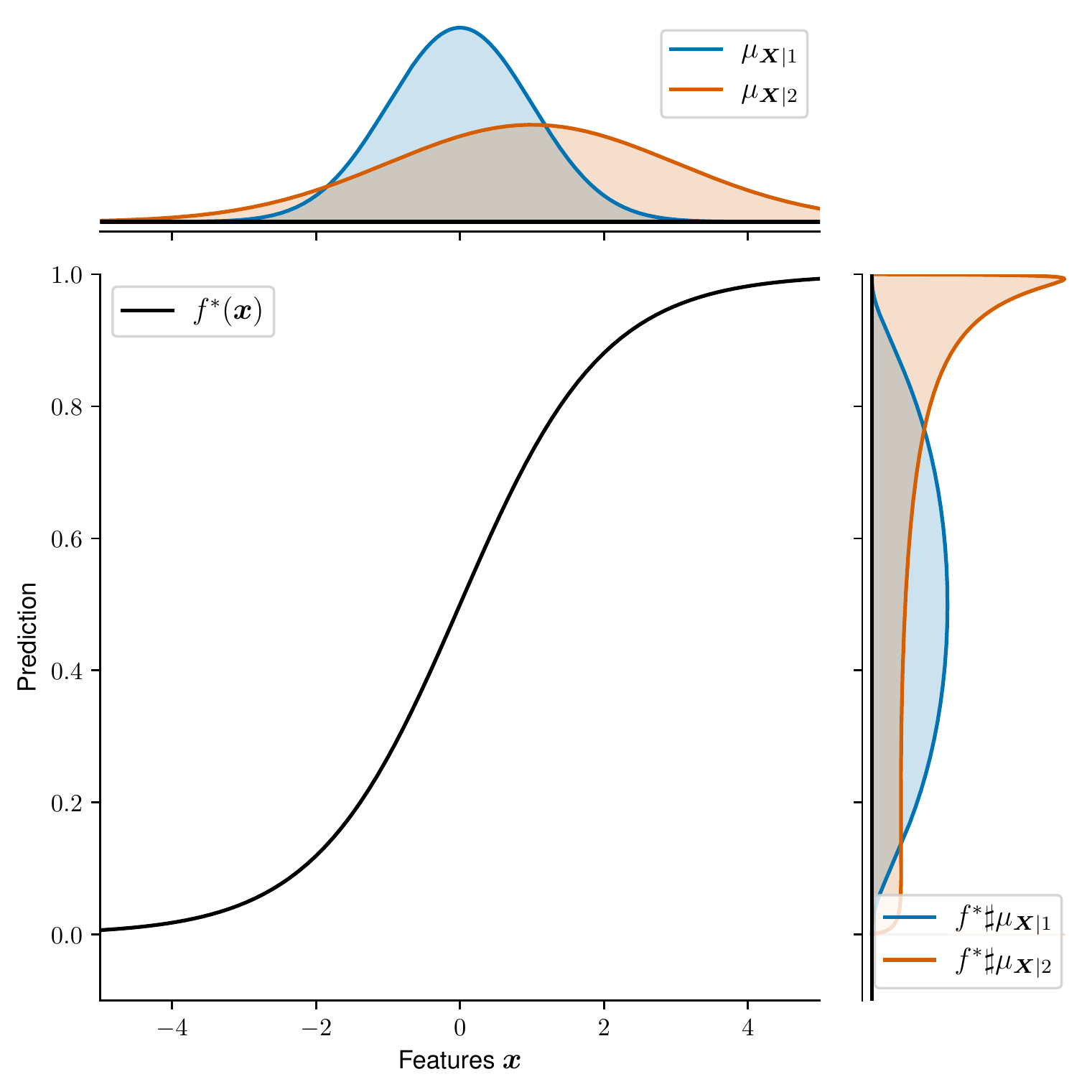}
\includegraphics[width=0.48\textwidth]{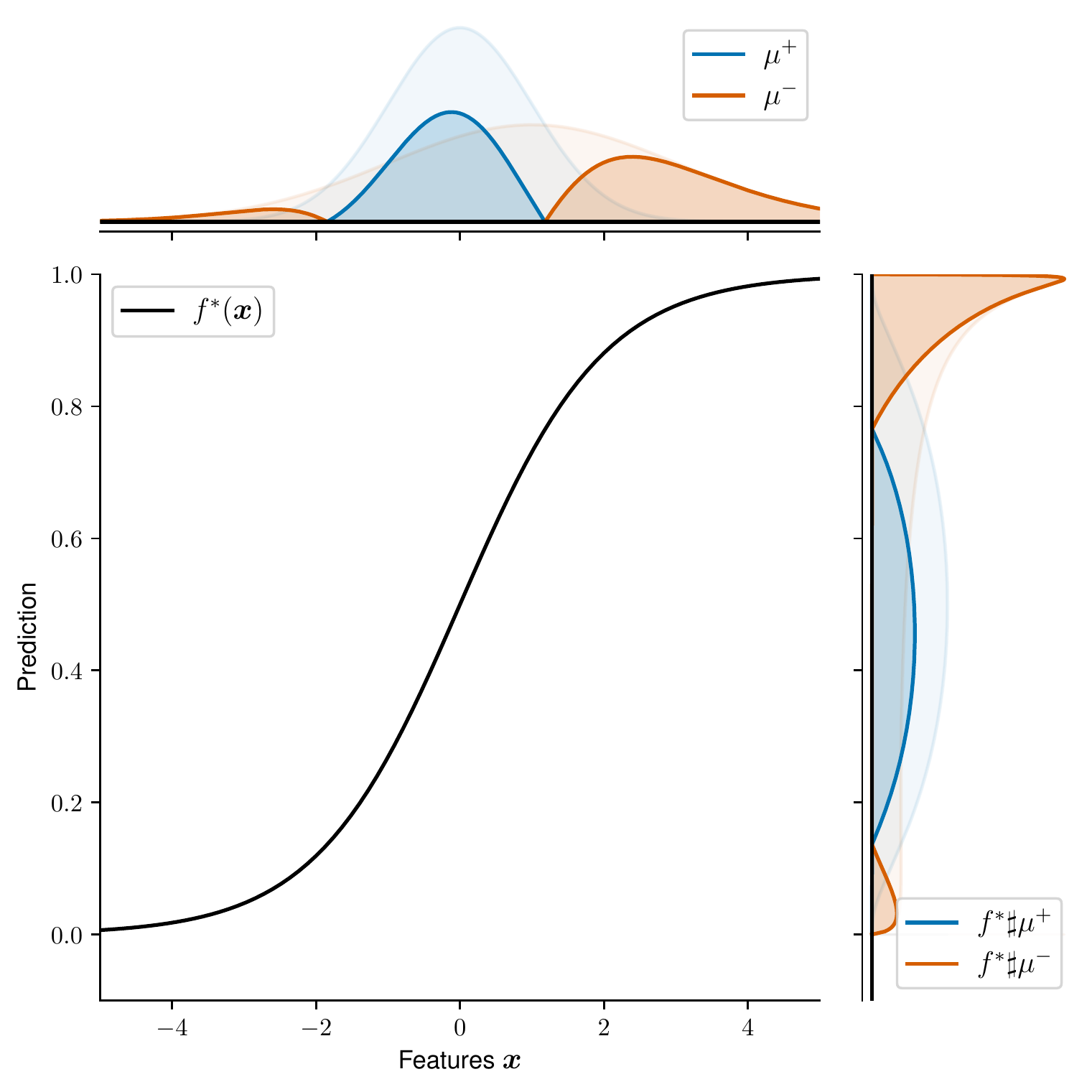}
\caption{(\texttt{Left}) The Bayes optimal prediction $f^*$ is illustrated in the center. The group-wise distributions of features is illustrated on the top. The group-wise distribution of the Bayes optimal prediction is illustrated on the right. (\texttt{Right}) Jordan decomposition of $\mu_{\bsX | 1} - \mu_{\bsX | 2}$ is illustrated on the top and their push-forward (through $f^*$) measures are on the right.}
\label{fig:moving}
\end{figure}

\paragraph{Fairness of prediction rules} Note that since the prediction rules are not allowed to depend on the sensitive attribute in its functional form, they do not produce Disparate Treatment. In order to show that the prediction rules defined in \eqref{eq:predictor} satisfy other fairness constraints, we make one standard technical assumption about particular distributions induced by the Bayes rule.
\begin{assumption}
    \label{ass:cont_signed}
    The measures $f^* \sharp \mu^+$,  $f^* \sharp \mu^-$ are non-atomic with finite second moments.
\end{assumption}

The following proposition states that, under the previous assumption, the defined prediction rules achieve Demographic Parity. 

\begin{proposition}
    \label{prop:DP_good}
   Let Assumption~\ref{ass:cont_signed} hold. Let $Q: [0, 1] \to \bbR$ be any continuous non-decreasing function, then the prediction rule $f_Q$ is fair in the sense of Demographic Parity.
\end{proposition}
The proof of Proposition~\ref{prop:DP_good} is postponed to Section~\ref{sec:proofs}. The result becomes rather intuitive following the interpretation of $Q$ as a quantile function of some continuous univariate probability measure $\lambda$ and of $Q \circ F_\square \circ f^*(\cdot)$ as the optimal transport map from $f^* \sharp P\mu^\square$ to $\lambda$ in combination with Lemma~\ref{lem:signed}.

We have a large class of prediction rules which avoid Disparate Treatment and achieves Demographic Parity. Can we find a subset of this class such that its elements also satisfy Equality of Group-Wise Risks? The next proposition explicitly gives a continuous non-decreasing function $Q^*$ such that the resulting prediction rule $f_{Q^*}$ satisfies the Equality of Group-Wise Risks constraint.
\begin{proposition}
\label{prop:EGWR_plus}
Let Assumption~\ref{ass:cont_signed} hold. For the choice $Q^* =(F_+^{-1} + F_-^{-1})/2$, the prediction rule $f_{Q^*}$ is fair in the sense of Equality of Group-Wise Risks.
\end{proposition}

\begin{figure}[!t]
\centering
\includegraphics[width=0.48\textwidth]{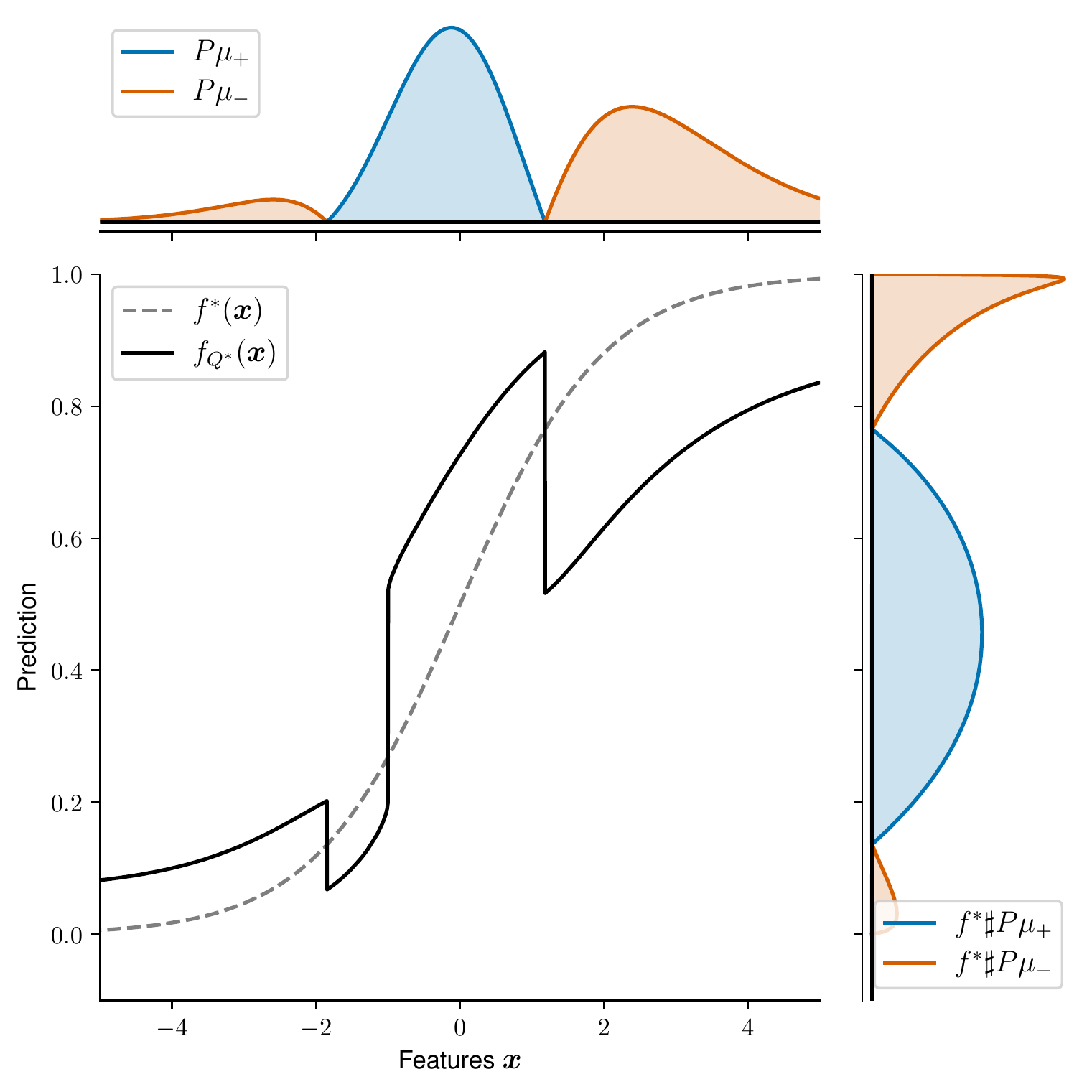}
\includegraphics[width=0.48\textwidth]{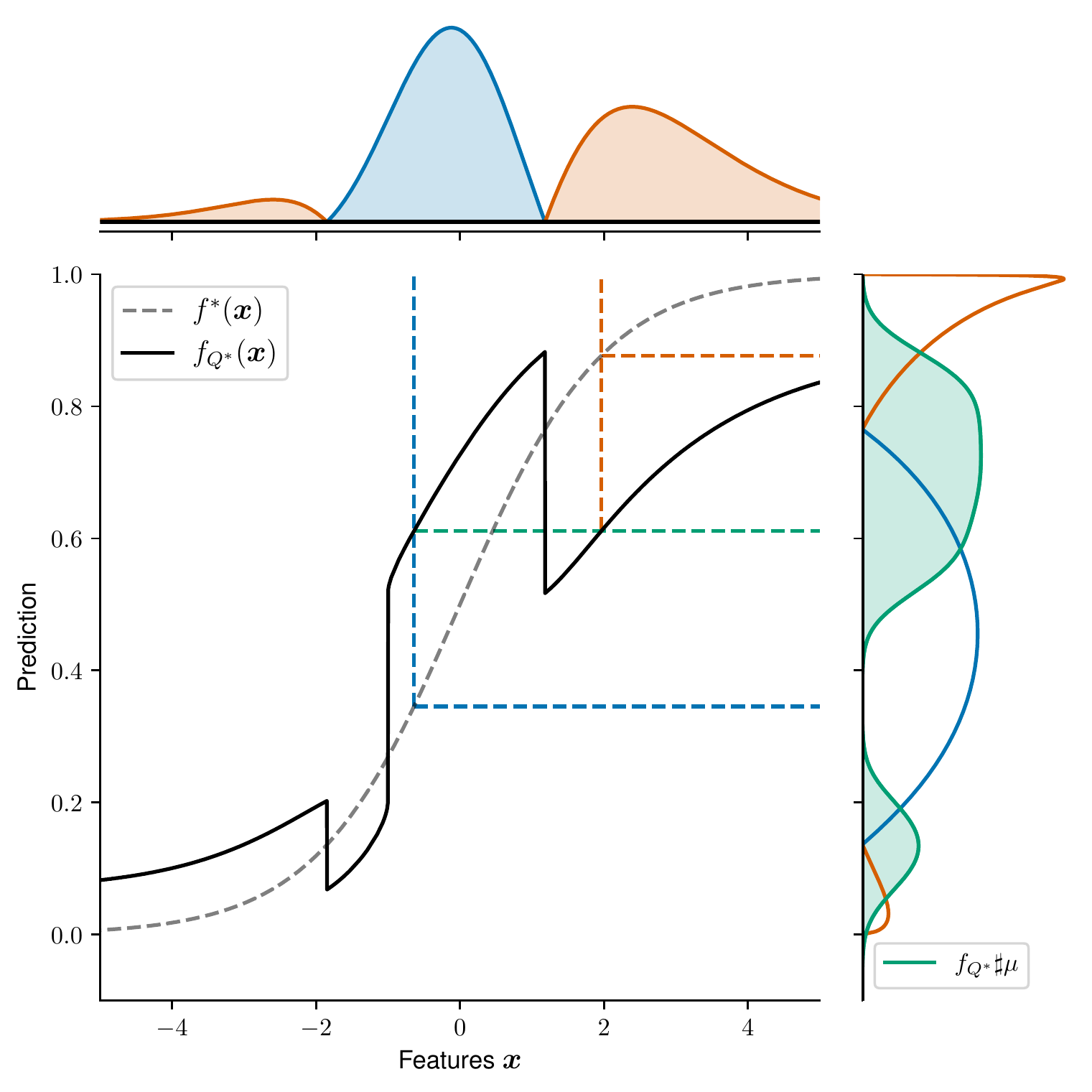}
\caption{(\texttt{Left}) Prediction $f_{Q^*}$ which achieves DP and EGWR. (\texttt{Right}) Concrete examples of predictions.}
\label{fig:moving2}
\end{figure}

\section{Discussion and open questions}
\label{sec:discussion}

In the previous section we have proved that the prediction rules defined in \eqref{eq:predictor} achieve Demographic Parity and that for a specific choice of continuous non-decreasing function $Q^*$, the prediction rule $f_{Q^*}$ also satisfies the Equality of Group-Wise Risks. The latter prediction rule is represented in Figure~\ref{fig:moving2} for a particular problem: the features are assumed to be group-wise Gaussian random variables with different means and variances. We set the Bayes optimal predictor as $f^*(\bsx) = 1/(1+e^{a\bsx})$ for some positive real $a>0$.

In both plots the dashed grey curve corresponds to the Bayes prediction rule $f^*(\bsx) = \mathbb{E}[Y \mid \bsX=\bsx]$ while the black solid curve represents the prediction rule $f_{Q^*}$ defined in Proposition~\ref{prop:EGWR_plus}. On top of the plots are the densities from the (normalized) Jordan decomposition of $\mu_{\bsX|1} - \mu_{\bsX|2}$ (see also Figure~\ref{fig:hahn}) while on the right side are the densities corresponding to the predictions.

In the right plot, the green horizontal dashed line corresponds to the prediction by $f_{Q^*}$ for the points whose axis correspond to the vertical blue and orange dashed lines. The horizontal blue and orange dashed lines correspond to the prediction by $f^*$. We notice that the prediction curve corresponding to $f_{Q^*}$ looks like a piece-wise translation of the Bayes decision rule in which the predicted value is increased for features which seem to come from the group corresponding to $S=1$ and lowered for the other features.

The prediction rule $f_{Q^*}$ could be formally considered as a good fair predictor since it simultaneously satisfies several formal group-fairness constraints and avoids Disparate Treatment. However, Demographic Parity and EGWR only define fairness on the group level and inspecting the individual level reveals a critical flow of this prediction rule. 
We have constrained our predictors to those that do not produce Disparate Treatment by prohibiting them from having the sensitive variable as direct input. Nevertheless, enforcing group level fairness constraints (such as DP and EGWR) forces the prediction rule to guess the sensitive attribute corresponding to a given feature vector $\bsx$. The idea of our prediction rules is simple: if a feature vector $\bsx$ is more likely to belong to some group then it is treated as a member of this group.
A critical resulting issue of this is that an individual from the minority (\ie the group which gets discriminated) which "looks like" an individual from the majority will be treated as the latter and thus might potentially receive a negative discrimination, worsening their position in the population and in the society. This is clearly contrary to what one would expect from a fair decision-making system and should therefore be avoided.
We remark that a simple remedy from the above flaw is to allow to construct a separate prediction rule for each sensitive group -- wave away the Disparate Treatment requirement. Indeed, making separate predictions for separate groups erases the effect of group guessing and allows to make a more informed decision~\citep{gouic2020price,chzhen2020fair,lipton2018does}.

An interesting open question concerns the optimality of the derived prediction rules: is it possible to find a prediction rule which avoids Disparate Treatment while achieving Demographic Parity and which has smaller squared risk than those of the prediction rules in \eqref{eq:predictor} ? An answer to this question would yield an important step towards understanding the limits of predictions under fairness constraint without having access to the sensitive attribute.
Establishing the optimality would also allow to address relaxed notions of fairness in this context and provide a statistical study similar to~\citet{chzhen2020minimax}.

\section{Conclusion}
\label{sec:conclusion}
In this work we proposed a large family of prediction rules which simultaneously avoid Disparate Treatment and achieve Demographic Parity. In addition, we also showed that a particular member of the proposed family equalizes the group-wise risks.
However, despite these fruitful formal fairness properties, none of the above predictions are able to comply with the intuitive understanding of fairness.
We attribute this effect to the avoidance of Disparate Treatment.
An interesting mathematical challenge which remains unsolved is connected with the risk optimality of the proposed prediction rules.

\section{Omitted proofs}
\label{sec:proofs}
We recall that the \emph{Wasserstein-2} distance between probability distributions $\mu$ and $\nu$ in $\mathcal{P}_2(\mathbb{R}^d)$, the space of measures on $\mathbb{R}^d$ with finite second moment, is defined as
\begin{align}\label{eq:Wasserstein2}
    \mathsf{W}_2^2(\mu, \nu) \coloneqq \inf_{\gamma \in \Gamma(\mu, \nu)} \left\{ \int_{\mathbb{R}^d \times \mathbb{R}^d} \lVert \bsx - \bsy \rVert_2^2 d\gamma(\bsx, \bsy)\right\}\enspace,
\end{align}
where $\Gamma(\mu, \nu)$ denotes the collection of measures on $\mathbb{R}^d \times \mathbb{R}^d$ with marginals $\mu$ and $\nu$.
See \cite{santambrogio2015optimal,villani2003topics} for more details about Wasserstein distances and optimal transport.

\begin{proof}[Proof of Proposition~\ref{prop:DP_good}]
In order to prove that $f_Q$ satisfies the Demographic Parity constraint we examine the following quantity
\begin{align*}
   \Delta_{f_Q}(t) = \mu_{\bsX | 1}\enscond{\bsx \in \bbR^p}{f_Q(\bsx) \leq t} - \mu_{\bsX | 2}\enscond{\bsx \in \bbR^p}{f_Q(\bsx) \leq t}\enspace,
\end{align*}
for any $t \in \bbR$. Fix some $t \in \bbR$.
Let us fix some $Q : [0, 1] \to \bbR$ continuous non-decreasing function. For simplicity we drop the subscript $Q$ from $f_Q$ and write $f$ instead.
We can write by the definition of $f$, $\mu$ and $\mu^+, \mu^-$ that
\begin{align*}
    \Delta_f(t)
    &=
    \int_{{f(\bsx) \leq t}} d\mu(\bsx)
    =
    \int_{{f(\bsx) \leq t}} d\mu^+(\bsx) - \int_{{f(\bsx) \leq t}} d\mu^-(\bsx)\\
    &=
    \int_{{Q\parent{\mu^+\enscond{\bsx' \in \bbR^p}{f^*(\bsx') \leq f^*(\bsx)}} \leq t}} d\mu^+(\bsx) - \int_{{Q\parent{\mu^-\enscond{\bsx' \in \bbR^p}{f^*(\bsx') \leq f^*(\bsx)}} \leq t}} d\mu^-(\bsx)\enspace.
\end{align*}
Let $Q^{-1}$ be the generalized inverse of $Q$. Hence, since $Q$ is assumed to be continuous we can write
\begin{align*}
    \Delta_f(t) =  \int_{{\mu^+\enscond{\bsx' \in \bbR^p}{f^*(\bsx') \leq f^*(\bsx)} \leq Q^{-1}(t)}} d\mu^+(\bsx) - \int_{{\mu^-\enscond{\bsx' \in \bbR^p}{f^*(\bsx') \leq f^*(\bsx)} \leq Q^{-1}(t)}} d\mu^-(\bsx)\enspace.
\end{align*}
Introduce $F_{\square}(\cdot) = \mu^{\square}\enscond{\bsx' \in \bbR^p}{f^*(\bsx') \leq \cdot}$ for $\square \in \{\pm\}$ and note that thanks to Assumption~\ref{ass:cont_signed} both $F_{+}$ and $F_{-}$ are non-decreasing continuous.
Thus,
\begin{align*}
    \Delta_f(t)
    &=
    \int_{F_+(f^*(\bsx)) \leq Q^{-1}(t)} d\mu^+(\bsx) - \int_{F_-(f^*(\bsx)) \leq Q^{-1}(t)} d\mu^-(\bsx)\\
    &=
    \int_{f^*(\bsx) \leq F_{+}^{-1} \circ Q^{-1}(t)} d\mu^+(\bsx) - \int_{f^*(\bsx) \leq F^{-1}_- \circ Q^{-1}(t)} d\mu^-(\bsx)\\
    &=
    F_+ \circ  F_{+}^{-1} \circ Q^{-1}(t) - F_- \circ F^{-1}_- \circ Q^{-1}(t) = 0\enspace.
\end{align*}
The proof is concluded since $\sup_{t \in \bbR} |\Delta_f(t)| = 0$ implies that $f$ satisfies the Demographic Parity constraint.
\end{proof}

\begin{proof}[Proof of Proposition~\ref{prop:EGWR_plus}]
In this proof we consider the prediction rule $f_{Q^*}$ defined in \eqref{eq:predictor} with the specific choice $Q^*\coloneqq(F_+^{-1} + F_-^{-1})/2$. Let $p_1 = \Prob(S=1)$ and $p_2= \Prob(S=2)=1-p_1$.
Since $Q^*$ is fixed in throughout this proof, we drop the subscript $Q^*$ and write $f$ instead of $f_{Q^*}$ for compactness.

Recall that we defined the signed measure $\mu = \mu_{\bsX|1} - \mu_{\bsX|2}$. Using its Hahn decomposition, $\mu = \mu^+ - \mu^-$, we can write $\mu_{\bsX|1} = \mu^+ - \mu^- + \mu_{\bsX|2}$ and express the risk of the predictor $f$ as
   \begin{align}
   \label{eq:risk_decomposition}
       \risk(f)
       &=
       p_1\int(f^*(\bsx) - f(\bsx))^2d\mu_{\bsX | 1}(\bsx) + p_2\int(f^*(\bsx) - f(\bsx))^2d\mu_{\bsX | 2}(\bsx) \nonumber \\ 
       &= \int(f^*(\bsx) - f(\bsx))^2d\mu_{\bsX | 2}(\bsx) \\ &\qquad+ p_1\parent{\int(f^*(\bsx) - f(\bsx))^2d\mu^+(\bsx) -  \int(f^*(\bsx) - f(\bsx))^2d\mu^-(\bsx)}\enspace. \nonumber
   \end{align}
   Since $f\sharp\mu^\square = T_\square\sharp(f^*\sharp\mu^\square)$ for $\square \in \{\pm\}$, where $T_\square = Q \circ F_\square$ is a monotone non-decreasing function, \cite[Theorem 2.9]{santambrogio2015optimal} implies
\begin{align*}
   \int(f^*(\bsx) - f(\bsx))^2d\mu^\square(\bsx) = \mathsf{W}_2^2(f^*\sharp\mu^\square, f\sharp\mu^\square), \text{ for } \square \in \{\pm\} \enspace.
\end{align*}
Following \cite[Section 6.1]{agueh2011barycenters}, the solution to the Wasserstein-2 barycenter problem
\begin{align*}
    \min_{\nu} \left(\frac{1}{2}\mathsf{W}_2^2(\nu, f^* \sharp \mu^+) + \frac{1}{2}\mathsf{W}_2^2 (\nu, f^* \sharp\mu^-)\right)
\end{align*}
is given by the measure
\begin{align}
\label{eq:1}
    \bar{\nu} &= \frac{1}{2}\left(F_+^{-1} +  F_-^{-1}\right) \circ F_+ \circ f^* \sharp \mu^+ \\
    \label{eq:2}
    &= \frac{1}{2}\left(F_+^{-1} +  F_-^{-1}\right) \circ F_- \circ f^* \sharp \mu^-\enspace.
\end{align}
Indeed, observe that $\frac{1}{2}\left(F_+^{-1} +  F_-^{-1}\right) \circ F_+$ is the optimal transportation plan from $f^* \sharp \mu^+$ to the barycenter of $f^* \sharp \mu^+, f^* \sharp \mu^-$.
Since Eq.~\eqref{eq:1} corresponds to $f\sharp \mu^+$ on $\supp(\mu^+)$ and Eq.~\eqref{eq:2} to  $f\sharp \mu^-$ on $\supp(\mu^-)$, the distances to the barycenter being equal, we have
\begin{align}
\label{eq:eq_dist_barycenter}
    \mathsf{W}_2^2(f^*\sharp\mu^+, f\sharp\mu^+) =     \mathsf{W}_2^2(f^*\sharp\mu^-, f\sharp\mu^-)\enspace.
\end{align}
Plugging \eqref{eq:eq_dist_barycenter} in \eqref{eq:risk_decomposition} yields
\begin{align*}
   \risk(f) = \int(f^*(\bsx) - f(\bsx))^2d\mu_{\bsX | 2}(\bsx)\enspace,
\end{align*}
and concludes the proof.
\end{proof}

\bibliographystyle{apalike}
\bibliography{biblio.bib}

\end{document}